\documentclass{article}
\usepackage{times}
\usepackage{graphicx} 

\usepackage[authoryear,round]{natbib}

\usepackage{amsmath}
\usepackage{amsthm}
\usepackage{authblk}

\newcommand{\Xbf}{\mathbf{X}}

\newcommand{\Vbf}{\mathbf{V}}
\newcommand{\Abf}{\mathbf{A}}
\newcommand{\Bbf}{\mathbf{B}}
\newcommand{\Mbf}{\mathbf{M}}
\newcommand{\Ibf}{\mathbf{I}}

\newcommand\numberthis{\addtocounter{equation}{1}\tag{\theequation}}

\newcommand{\tr}{\textrm{tr}}

\newcommand*{\approxdist}{\mathrel{\vcenter{\offinterlineskip
\vskip-.25ex\hbox{\hskip.55ex$\cdot$}\vskip-.25ex\hbox{$\sim$}
\vskip-.5ex\hbox{\hskip.55ex$\cdot$}}}}

\newtheorem{corollary}{Corollary}
\newtheorem{prop}{Proposition}

\newcommand{\proglang}[1]{\textsf{#1}}
\newcommand{\strong}[1]{{\normalfont\fontseries{b}\selectfont #1}}
\newcommand{\pkg}[1]{\strong{#1}}

\newcommand{\norm}[1]{\left|\left|#1\right|\right|}
\newcommand{\code}[1]{\mbox{\texttt{#1}}}

\begin{document} 

\title{Scatter Matrix Concordance: A Diagnostic for Regressions on Subsets
of Data}

\author[1]{Michael J. Kane}
\author[2]{Bryan Lewis}
\author[3]{Sekhar Tatikonda}
\author[4]{Simon Urbanek}

\affil[1]{Yale University and Phronesis LLC, New Haven CT, USA}
\affil[2]{Paradigm4, Waltham MA, USA}
\affil[3]{Yale University, New Haven CT, USA}
\affil[4]{AT\&T Labs Research, New York NY, USA}

\maketitle

\begin{abstract}

Linear regression models depend directly on the design matrix and its
properties.  Techniques that efficiently estimate model coefficients by
partitioning rows of the design matrix are increasingly popular for large-scale
problems because they fit well with modern parallel computing architectures.
We propose a simple measure of {\em concordance} between a design matrix 
and a subset
of its rows that estimates how well a subset captures the variance-covariance
structure of a larger data set.  We illustrate the use of this measure in a
heuristic method for selecting row partition sizes that balance statistical and
computational efficiency goals in real-world problems.

\end{abstract} 

\section{Introduction}
\label{sect:intro}

A common procedure in supervised learning problems when the number of rows of
data is large and far exceeds the number of columns is to partition the rows,
fit models on individual partitions, and combine them by averaging or other
aggregation into a single model. This computational approach is referred to as
``Divide and Recombine'' (D\&R) in \cite{Guha2012} and has gained wide use in
part because it can be described as a single MapReduce \citep{Dean2008} step
and is easily implemented in software frameworks like Hadoop
\citep{ApacheHadoop} and Spark \citep{ApacheSpark}.

Partitions are constructed either by {\em conditioning-variable division} or
{\em replicate division}. The former adds samples to a partition based on one
or more of the variables in the data. For categorical variables this generally
means one partition per level and is equivalent to a variable interaction
between the categorical variable being conditioned on all other regressors in
the model. For continuous variables this means one partition per range of
values.  Alternatively, data can be partitioned along multiple
conditioning variables. Replicate division creates partitions using random
sampling without replacement.

\cite{Guha2012} provide conditions such that a model constructed by averaging
coefficients from ordinary least squares models along replicate division data
partitions converges asymptotically to the single ordinary least squares model
fitted to all of the data, called the {\em reference model}.  We constrain our
attention to D\&R under replicate division to create a single, averaged model.


Note that the D\&R method is related to boosting since in both cases an
ensemble of models is created by sampling from the original data set. There are
differences, however. First, when boosting samples with replacement the D\&R
averaged model samples without replacement.  Second, where the result of 
training a boosted
model is an ensemble of individual learners the result of training a D\&R
model is a single model where parameters have been averaged together.

D\&R describes a general computational approach to large scale regression and
classification requiring only a single pass through the data set. This
gives D\&R a computational advantage over those models
which require multiple passes through the data. 
\cite{Matloff2014} extends both the statistical results by providing a broader
class estimators under which D\&R converges 
as well as the complexity analysis 
by showing that the computational speed up can be greater than
the number of parallel processes applied to each division. 
\cite{Kleiner2011} also developed the statistical theory further by
identifying a general class of models that converge under the random partition
assumption as well as showing that bootstrap models can be created on each
partition and the ensemble of models over all partitions converges
asymptotically to a single bootstrap model over the entire data set.

There are at least two practical issues to consider when implementing the D\&R
approach to model fitting with replicate division on distributed data.  First,
partitions should consist of random sets of rows to ensure that the aggregate
model converges to its reference. This is important in real-world, finite-sized
problems to avoid artifacts related to ordering in the data, for example from
the data collection process.  If samples are arranged in partitions, without
randomization, then individual partitions may capture information that is
drastically different from the other partitions or the total population. As a
result, an aggregate model can perform poorly when compared to its reference.

The second challenge is in deciding the number of samples per block. The
convergence results show that the ensemble of models converges to the reference
model asymptotically in the number of random samples in the partitions.  This
implies that, to minimize the difference between the ensemble and the
reference, block sizes should be as large as possible. However, the amount of
speed-up achieved is generally directly proportional to the number of partitions
(assuming each block is allocated its own process). Therefore, there is a
trade-off between the statistical consistency of a model created using the D\&R
approach (with respect to its reference) and the parallelism that can be
achieved when fitting the model. 

The regression techniques presented work on subsets of disjoint 
rows and, with this observation, a natural subsequent
question is how well do these individual subsets represent the data set
as a whole? This is a relevant question 
since if we have a small representative subset then we may not need all
of the data. We may be able to fit a model more efficiently 
by only using the representative subset. If we find that some subsets
differ drastically from the rest of the data set we may need to investigate
these subsets further. They may indicate different underlying relationships
that can be conditioned upon or they may indicate data integrity issues. In
either case they motivate further investigation.

This paper explores the trade-off between statistical consistency and
computational efficiency by introducing a statistic that estimates the
concordance between a subset and the entire data set based on their
respective variance-covariance structures.  We provide reproducible experiments
that illustrate our concept of concordance and its use along with benchmark
results. A final section is devoted to discussion and future directions.

\section{Motivating Case: Least Squares Regression} \label{sect:motivation}

Consider the following example linear model and corresponding least
squares problem:
\begin{equation} \label{eqn:linear_model}
Y = \Xbf \beta + \varepsilon,
\end{equation}
where $Y, \varepsilon \in \mathcal{R}^n$, and $\beta\in\mathcal{R}^d$, $n\ge d$;
each element of $\varepsilon$ is an i.i.d.
random variable with mean zero; and 
$\Xbf$ is a matrix in $\mathcal{R}^{n \times d}$ with full column rank.
The ordinary least squares
problem is posed as $\widehat \beta = \text{argmin}_\beta 
  \norm{\Xbf \beta - Y}^2$
and has the closed form analytic solution defined by the normal equations:
\begin{equation} \label{eqn:lm_closed_form}
\widehat \beta = \left( \Xbf^T \Xbf \right)^{-1} \Xbf^T Y.
\end{equation}
We remark that computed solutions rarely use Equation \ref{eqn:lm_closed_form}
directly but rather use QR or SVD decompositions of $\Xbf$ for numerical
stability. Equation \ref{eqn:lm_closed_form} remains important for analysis
purposes. Consider the row-wise partitioning of Equation \ref{eqn:linear_model}:
\begin{equation*}
\left[
  \begin{array}{c} Y_1 \\ Y_2 \\ \vdots \\ Y_r \end{array} \right] = 
  \left[ \begin{array}{c} \Xbf_1 \\ \Xbf_2 \\ \vdots \\ \Xbf_r
\end{array}
\right]
  \beta + \left[ \begin{array}{c} \varepsilon_1 \\ \varepsilon_2 \\ \vdots \\
                 \varepsilon_r \end{array} \right],
\end{equation*}
where
$Y_1, Y_2, ..., Y_r$, $\Xbf_1, \Xbf_2, ..., \Xbf_r$ and 
$\varepsilon_1, \varepsilon_2, ...,  \varepsilon_r$ are 
data partitions such that each block is composed of subsets of
rows of a data set, blocks are disjoint, and the aggregate of all
blocks is the original data set.  Without loss of generality we
assume that $n/r=c$ is an integer so that the number of samples in 
each submatrix is the same. The blocks may, for example, be located in files
across a network of computers for distributed computation.

When each of the $\Xbf_i$'s represent a random partition of $\Xbf$ then
the estimate for least squares regression coefficients using D\&R averages 
the block-wise estimates of the slope coefficients $\tilde \beta_i$,
\begin{equation}\label{blockwise_dnr}
\tilde{\beta}_i = \left(\Xbf_i^T \Xbf_i \right)^{-1} \Xbf_i^T Y_i, \qquad
\tilde{\beta} = \frac{1}{r} \sum_{i=1}^r \tilde{\beta}_i.
\end{equation}
Compare the D\&R approach to block-wise computation of
the (full) least squares solution:
\begin{equation}\label{blockwise_full}
\hat{\beta} = \left(\sum_{i=1}^r \Xbf_i^T \Xbf_i \right)^{-1}
\sum_{i=1}^r \Xbf_i^T Y_i.
\end{equation}
Both approaches are similarly easy to compute in parallel. The overall
computational cost of both approaches is $O(d^3)$, although the D\&R
approach (Equation \ref{blockwise_dnr}) has a larger constant term by a factor 
of $r$
compared to block-wise solution of the full problem 
(Equation \ref{blockwise_full}).

Recall that the D\&R solution $\tilde{\beta}$ is an estimate of the true
solution $\hat{\beta}$.
It's reasonable to ask the why one would spend {\em more} computational effort
to produce only an estimate of a solution, when a similarly easy direct
solution method is available. Despite the apparent advantage of the direct
block-wise solution method shown in Equation \ref{blockwise_full}, 
the D\&R approach is
potentially superior in two ways: lower network communication cost in parallel
computing settings and better numerical stability. We outline each advantage
below.

{\em Lower network communication cost.} Assume that the problem is distributed
so that each data block $i=1,2,\ldots,r$ is located on a different computer in
a network.  The D\&R approach shown in Equation \ref{blockwise_dnr} 
averages $r$ sets of
$d$ model coefficients to produce an averaged output model, for a total $rd$
numbers transmitted between computers over the network. Block-wise solution of
the full problem outlined in \ref{blockwise_full}, by comparison, sums $r$
partial $d\times d$ matrix products and $r$ vectors of length $d$, for a total
of $rd^2 + rd$ numbers to transmit over the network. The $d^2$ communication
cost for the full solution is expensive when there are moderately large numbers
of columns in the model matrix. For example with 8-byte double precision
floating point numbers, $r=100$ and $d=1000$, the full problem solution must
transmit $800\,$MB across the network. The D\&R approach only transmits
$800\,$KB by comparison.

{\em Better numerical stability.} Although we routinely use the normal
equations for analysis of least squares problems in exact arithmetic, their use
computationally is not recommended because the matrix $\Xbf^T\Xbf$ is generally less
well-conditioned (and never better conditioned) than $\Xbf$.  Instead, least
squares problems are typically computed using either a QR or singular value
decomposition of the model matrix $\Xbf$.  Unfortunately, parallel computation
of a QR or SVD factorization for block row-distributed matrices is neither easy
nor readily available to most high-level programming languages like R and
Python, especially in MapReduce-like computing settings\footnote{A notable
exception is HPC systems using ScaLAPACK and MPI--see the R {\bf pbd} package,
for example--but these are usually very specialized systems.}. Note also that
model matrices involving contrast variables derived from so-called factor
variables are often sparse. SVD and QR decompositions destroy sparsity patterns
and the resulting factorized model matrix may consume much more memory than the
original.  Distributed computation of the full problem is typically performed
using the normal equations as shown in Equation \ref{blockwise_full} for 
these reasons.
The D\&R approach, by contrast, is free to use numerically stable solution
methods in each block. Moreover, since the blocks are relatively small, loss of
a sparse model matrix representation in each block due to factorization is more
tolerable.

We illustrate our point about numerical stability
with a simple example. 
\begin{equation*}
X = 
\left(
  \begin{array}{cc}
    10^9 & -1\\
    -1  & 10^{-5}\\
  \end{array}
\right),\qquad
\beta = \left(\begin{array}{c} 1\\1\\ \end{array}\right),
\qquad\mbox{ and },\qquad
y = X\beta.
\end{equation*}
Note that $X$ is an ill-conditioned matrix, but not so ill-conditioned 
to prevent
numerically-stable techniques from working.  This can be demonstrated
using the R programming environment \citep{Rcore} to compare least square 
solutions of this example
computed using a stable technique and using the normal
equations.
\begin{flushleft}
\begin{verbatim}
> X <- matrix( c(1e9,-1,-1,1e-5), 2)
       [,1]   [,2]
[1,]  1e+09 -1e+00
[2,] -1e+00  1e-05

> y <- X %*% c(1,1)
\end{verbatim}

One stable least-squares solution approach gives us
the expected solution:

\begin{verbatim}
> qr.solve(X,y)
     [,1]
[1,]    1
[2,]    1
\end{verbatim}

Now with the normal equations. Note that in \proglang{R} we
form the inverse of $X^T X$ as described in Equation \ref{eqn:lm_closed_form}. 

\begin{verbatim}
> qr.solve(t(X) %*% X) %*% t(X) %*% y
Error in qr.solve : singular matrix 'a' in solve
\end{verbatim}

An informative error. If we override the error, 
we get a bad result that is very different
from what we expected:

\begin{verbatim}
> qr.solve(t(X) %*% X, tol=0) %*% t(X) %*% y
             [,1]
[1,]    0.9999995
[2,] 1080.4998042
\end{verbatim} 
\end{flushleft}

The example, although pathological by design, shows that using the normal
equations directly to solve least squares problems can lead to failure or 
poor results.
Even in better-conditioned problems, use of the normal equations will
result in a loss of numerical accuracy in the solution.  The D\&R method is
free to employ numerically-stable techniques in each block.

The normal-equation approach does have the advantage of being easily
updatable.  data are received they can be added to a new
block. If the block-coefficients are stored then updating the model is simply a
matter of getting the estimate for the new block and once again averaging over
all of the coefficient estimates. The removal of data or the updating of an
existing block can be handled similarly. The D\&R method is only updatable
when the incoming data are guaranteed to be distributed at random, that
is there is no local correlation structure that exists only a new block.
The effect of non-randomness is explored further in 
Section \ref{sect:nonrandom}.





\section{Common-basis Concordance}

The previous section shows the trade-offs and characteristics of estimating
the slope coefficients in a ordinary linear regression using the 
normal equations along with D\&R. In both cases the slope coefficients are 
a function
of the variance-covariance structure among regression and regressor 
variables and it seems reasonable that fitting subsets a subset of data
could give similar results and require less computational complexity.
This section introduces concordance as a statistic for capturing
how well a subset of data represents the whole. For this paper,
a data set is ``representative'' of a larger data set if the correlation
structure of the two data sets are similar and we are most interested
in the case where one data set is a small subset of the rows of a larger 
data set. In Section \ref{sect:conclusion} other potential
applications and directions are proposed.

Let $\Abf$ and $\Bbf$ be $n \times d$ and $m \times d$ matrices respectively
The previous section is meant to motivate 
with $m,n\ge d$; let $\Abf^T \Abf$ and $\Bbf^T \Bbf$ have eigenvalue
decompositions $\Vbf^T \lambda_A \Vbf$ and $\Vbf^T \lambda_B \Vbf$,
$\Vbf^T\Vbf=\Ibf$, respectively; and assume $(\Abf^T \Abf)^{-1}$ and $(\Bbf^T
\Bbf)^{-1}$ exist.
The concordance of these two matrices is defined as
\begin{equation} \label{eqn:rel_stability}
S(\Abf, \Bbf) = \frac{n}{dm} \left| \left|\Abf \Bbf^\dagger \right| \right|^2_F,
\end{equation}
where $\Bbf^\dagger=(\Bbf^T\Bbf)^{-1}\Bbf^T$ is the Moore-Penrose matrix
pseudo-inverse of $\Bbf$, and $\|\cdot\|_F$ is the Frobenius matrix norm.
This measure essentially compares the
variance-covariance structure between two matrices.  A concordance value of one
results when $\Abf^T\Abf = \Bbf^T\Bbf$.
The concordance value is less than one when there is, on average, less variance
in $\Abf$ than $\Bbf$, and the concordance value is greater than one in the
reverse case.  Statistical characteristics of common-basis
concordance will be derived in the next section. The rest of this section
derives some deterministic characteristics.

\begin{prop} \label{prop:trace_equivalence}
If the common-basis concordance conditions specified above hold, then
\begin{equation} \label{eqn:trace_equivalence}
\left| \left| \Abf \Bbf^\dagger \right| \right|_F^2 =
  \sum_{i=1}^d \frac{ \lambda_A (i) }{ \lambda_B (i) }
\end{equation}
where $\lambda_A(i)$ and $\lambda_B(i)$ are the $i$th eigenvalues of 
$\Abf^T\Abf$ and $\Bbf^T\Bbf$ respectively.
\end{prop}
\begin{proof}
\begin{align*}
\left| \left| \Abf \Bbf^\dagger \right| \right|_F^2
  &= \tr \left( \Bbf \left(\Bbf^T \Bbf \right)^{-1} \Abf^T \Abf 
    \left(\Bbf^T \Bbf \right)^{-1} \Bbf^T \right) \\
  &= \tr \left( \Bbf^T \Bbf \left(\Bbf^T \Bbf \right)^{-1} \Abf^T 
    \Abf \left(\Bbf^T \Bbf \right)^{-1} \right) \numberthis \label{align:cyc_perm1} \\
  &= \tr \left( \Abf^T \Abf \left(\Bbf^T \Bbf \right)^{-1} \right) \numberthis \label{align:start_here}\\
  &= \tr \left( \Vbf^T \lambda_A \Vbf \Vbf^T \lambda_B^{-1} \Vbf \right) \\
  &= \tr \left( \Vbf \Vbf^T \lambda_A \lambda_B^{-1} \right) \numberthis \label{align:cyc_perm2} \\
  &= \tr \left(\lambda_A \lambda_B^{-1} \right)
\end{align*}
Steps \ref{align:cyc_perm1} and \ref{align:cyc_perm2} follow from the 
cyclic permutation property of the trace. 
\end{proof}

\begin{corollary} \label{cor:ortho_proj_equiv}
The matrix concordance remains unchanged if both $\Abf$ and
$\Bbf$ are right-multiplied by any orthonormal matrix, in particular
the eigenvector matrix $\Vbf$.
\end{corollary}
The corollary follows directly from
$\| \Abf \Vbf (\Bbf \Vbf)^\dagger \|_F^2 = 
\| \Abf \Vbf \Vbf^T \Bbf^\dagger \|_F^2 =  \| \Abf \Bbf^\dagger \|_F^2.$
Proposition \ref{prop:trace_equivalence} and Corollary
\ref{cor:ortho_proj_equiv} require that eigenvectors of the scatter matrices of
$\Abf$ and $\Bbf$ are the same to show that the concordance of their
variance-covariance structure only depends on their eigenvalues.


The matrix concordance for matrices sharing a common scatter matrix
eigenvector basis is the ratios of the eigenvalues of their scatter
matrices.  Applied to data-analytic statistical challenges, concordance between
two data matrices can be analyzed to see if both are samples from the same
distribution. Furthermore concordance allows us to estimate how well a sample
from a data set captures the variance-covariance structure of the entire data
set.

\section{Deriving the Ratio Distributions from the Concordance Statistic}

The previous section proposed a measure of concordance between two matrices
with the same variance-covariance structure based on the Frobenius norm of 
one matrix normalized by the pseudo-inverse of the other. Furthermore, it was 
shown that the proposed concordance measure is preserved when each of the two 
matrices is right-multiplied by the eigenvector matrix of the shared 
variance-covariance matrix. In this section, matrices will be assumed to
be drawn at random from a specified distribution and the concordance's
distribution will be derived.

In particular, it will be assumed that the data to be analyzed is drawn
from some distribution with zero mean and known 
variance-covariance matrix $\Sigma$.
A data set with $n$ samples will be denoted $\Xbf_{\left[n\right]}$ 
indicating that the data set is made up of the set of samples from
1 to $n$. By introducing this absolute index to the samples we can
easily express the first $i$ samples in $\Xbf_{\left[n\right]}$ as
$\Xbf_{\left[i\right]}$ for $i \leq n$. Likewise, we can express 
all of the samples except the first $i$ as  $\Xbf_{\left[n\right] \setminus
\left[i\right]}$. When the concordance is calculated between an entire
data set and a subset it will be referred to as {\em overlapping} and
when the data sets are disjoint they will be referred to as
{\em non-overlapping}.

The concordance $S\left(\Xbf_{\left[i \right]}, \Xbf_{\left[n\right]}
\right)$ normalizes the variance-covariance matrix of $\Xbf_{\left[i \right]}$
by $\Xbf_{\left[n\right]}$. By Corollary \ref{cor:ortho_proj_equiv} this
is equivalent to projecting the data onto the common orthonormal column
basis. The resulting orthonormalized variance-covariance matrix has 
zero expected values for all off-diagonal elements. The diagonal elements
are the ratios of the common variance estimates, which reduce to a
sum of random variables centered at one with standardized dispersion.

%


\begin{prop} \label{prop:beta_converge}
Suppose that $\Xbf_{\left[n\right]}$ is sampled from an
i.i.d. multivariate normal distribution.  Then for sufficiently large
$d$
\begin{equation*}
S \left(\Xbf_{\left[i\right]}, \Xbf_{\left[n\right]} \right) 
  \approxdist \mathcal{N}\left( 1, \frac{2(n-i)}{di(n+2)} \right)
\end{equation*}
where $\mathcal{N}$ is a normal distribution with specified mean
and variance parameters.
\end{prop}
\begin{proof}
The normed matrices in Equation \ref{eqn:rel_stability} reduce to $d$ sums of 
$n$ squared standard normals. Furthermore $i$ of the samples are repeated 
in $\Xbf_{\left[i\right]}$ and $\Xbf_{\left[n\right]}$.
Then the concordance can be expressed as:
\begin{align*}
S \left(\Xbf_{\left[i\right]}, \Xbf_{\left[n\right]} \right) 
&=\frac{1}{d} \sum_{j=1}^d 
\frac{ \frac{1}{i} \sum_{k=1}^i \Mbf[k,j]^2 }
  {\frac{1}{n} \sum_{k=1}^n \Mbf[k,j]^2 } \numberthis \label{align:sim_with_m}\\
  &=\frac{1}{d} \sum_{j=1}^d 
  \frac{ \frac{1}{i} \sum_{k=1}^i \sigma_j Z^2_k }
  {\frac{1}{n} \sum_{k=1}^n \sigma_j Z^2_k }  \\
  &=\frac{1}{d} \sum_{j=1}^d 
  \frac{ \frac{1}{i} \sum_{k=1}^i Z^2_k }
  {\frac{1}{n} \sum_{k=1}^n Z^2_k }  \numberthis \label{align:ratio}
\end{align*}
where $\Mbf = \Xbf_{[n]}\Vbf$ and
$\sigma_j$ is the $j$th eigenvalue of $\Sigma$ and $Z_k$ is distributed
as standard normal. Each of the $d$ terms in the summation are a ratio
$\chi^2$ random variables. The degrees of freedom in the numerator and
denominator are equal to $i$ and $n$ respectively with $i$ of the sampled
random variables appearing in both the numerator and denominator.
The ratio of $\chi^2$ distributions, where samples are repeated in the 
denominator, is distributed as Beta. The result follows by applying
the central limit theorem to the sample mean of $d$ independent Beta
distributions each of which are multiplied by the constant $n/i$.
\end{proof}

\begin{prop} \label{prop:f_converge}
Assume the same conditions as in Proposition \ref{prop:beta_converge} along with
the added condition that $n >> i >> 2$. Then 
\begin{equation*}
S \left(\Xbf_{\left[i\right]}, \Xbf_{\left[n\right] \setminus \left[i\right]} 
  \right) \approxdist \mathcal{N}\left(1, \frac{2n}{i}\right)
\end{equation*}
\end{prop}
\begin{proof}
By definition, the concordance is the same as in Equation \ref{align:sim_with_m}
except that the summation in the denominator goes from $i+1$ to $n$.
As a result we get a result similar to Equation \ref{align:ratio} where
\begin{equation}
S \left(\Xbf_{\left[i\right]}, \Xbf_{\left[n\right]} \right) =
  \frac{1}{d} \sum_{j=1}^d \frac{ \frac{1}{i} \sum_{k=1}^i Z^2_k }
  {\frac{1}{n-i} \sum_{k=i+1}^n Z^2_k }.
\end{equation}
The ratio of independent 
$\chi^2$ distributions, where the numerator and denominator are
normalized by their respective degrees of freedom, is $F$ with 
$i$ and $n-i$. The result follows by applying the central limit theorem
to the sample mean of $d$ independent $F$ distributions.
\end{proof}

\begin{prop} \label{prop:cauchy_converge}

Suppose that $\Xbf_{\left[ n\right]}$ is sampled from some distribution
with constant mean and variance-covariance.
Let $\Mbf$ be $\Xbf_{\left[ n\right]} \Vbf$ as before.
If $1 \leq j \leq d$, $Z$ is standard normal, and the following 
joint convergence holds
\begin{equation*} \label{eqn:joint_converge}
\left[ \frac{1}{n-i} \sum_{k=i+1}^{n} \Mbf[k,j]^2 , 
       \frac{1}{i} \sum_{k=1}^{i} \Mbf[k,j]^2  \right] \rightarrow_d
  \left[ 1 + \frac{\sigma_j}{\sqrt{n-i}} Z, 1+ \frac{\sigma_j}{\sqrt{i}} Z \right].
\end{equation*}
Then the concordance is approximately distributed as Cauchy with 
location parameter 1 and scale parameter $\sqrt{(n-i)/i}$.

\begin{equation}
S \left(\Xbf_{\left[i\right]}, \Xbf_{\left[n\right] \setminus \left[i\right]} \right) 
  \approxdist Cauchy\left(1, \sqrt{\frac{n-i}{i}}  \right)
\end{equation}
\end{prop}
\begin{proof}
\begin{align*}
S \left(\Xbf_{\left[i\right]}, \Xbf_{\left[n\right]} \right) 
&=\frac{1}{d} \sum_{j=1}^d 
\frac{ \frac{1}{i} \sum_{k=1}^i \Mbf[k,j]^2 }
  {\frac{1}{n-i} \sum_{k=i+1}^n \Mbf[k,j]^2 } \\
  &\rightarrow_d \frac{1}{d} \sum_{j=1}^d \frac{ \sigma_j Z/\sqrt{i} }
    { \sigma_j Z / \sqrt{n-i}} \numberthis \label{align:cmt}\\
  &= \frac{1}{d} \sum_{j=1}^d \frac{ \sqrt{n-i} Z } {  \sqrt{i} Z } 
\end{align*}
Each term in the summation is a ratio of two normal random variables.
Equation \ref{align:cmt} follows from the joint convergence assumption 
(Equation \ref{eqn:joint_converge}) and the application of the 
continuous mapping theorem. The 
ratio of two normally distributed random variables, with common location
parameter, is distributed as Cauchy. The result follows by realizing 
that the sample mean of independent Cauchy distributions, with common
location and scale parameters is Cauchy with the same location and
scale parameter.
\end{proof}

The scale and location values based on the distributional results are 
shown in Table \ref{table:models}.
It may be noted that, while the Cauchy approximation relies on fewer
distributional assumptions it is also less applicable. In particular
its support includes the negative reals. This is a result of 
taking the ratio of central limit theorem approximation of two random 
variables, as shown in Equation \ref{align:cmt}. However, this derivation
may suggest that, when the normal distribution assumptions do not hold
the resulting concordance distribution is heavy-tailed.

\begin{table}[h]
\begin{center}
\begin{tabular}{|c|c|c|c|} \hline
{\bf Model} & {\bf Location} & {\bf Scale} & \bf{Approx. Concordance} \\ 
& & & {\bf Distr.} \\ \hline 
$\frac{n}{i}Beta\left( \frac{i}{2}, \frac{n-i}{2}\right)$ & 1 & $\frac{2 \left( n-i \right)}{i \left( n+2 \right)}$ & 
  $\mathcal{N}\left(1, \frac{2(n-i)}{din} \right)$\\ \hline
$F\left(i, n-i\right)$ & $\frac{n-i}{n-i-2} $ & 
  $\frac{2\left(n-i\right)^2 \left(n-2\right)}{i\left(n-i-2\right)^2 \left(n-i-4\right)}$ & $\mathcal{N}\left( 1, \frac{2n}{di(n-i)} \right)$ \\ \hline
$Cauchy \left(1, \sqrt{\frac{n-i}{i}} \right)$ & 1 & $\sqrt{\frac{n-i}{i}}$ &
  $Cauchy \left(1, \sqrt{\frac{n-i}{i}} \right)$\\
\hline
\end{tabular}
\end{center}
\caption{The derived model distributions with their location and scale 
parameters. Note that for the $F$ and Beta model the approximate concordance
assumes $n >> i >> 2$.}
\label{table:models}
\end{table}

\section{Benchmark Description, Design, and Implementation}

To assess the behavior of the relative stability measures proposed
in Equation \ref{eqn:rel_stability} 
this section makes use of the ``Airline on-time performance'' data
set \citep{AirlineDataSet}, which was released for the 2009 
American Statistical Association (ASA) Section on Statistical 
Computing and Statistical Graphics biannual data exposition.
The data set includes commercial flight arrival and departure information 
from October 1987 to April 2008 for those carriers with at least 1\% of
domestic U.S. flights in a given year. In total, there is
information for over 120 million flights, with 29 variables related to flight 
time, delay time, departure airport, arrival airport, and so on. 
In total, the uncompressed data set is 12 gigabytes (GB) in size.
Benchmarks in this section focus on the model matrix representation of 
the of the variables shown below in Table \ref{tab:cov_vars}.
The model matrix under consideration will use the treatment-contrast 
expansion of the categorical variables and has a total of 43 columns ($d=43$).

\begin{table}[h]
\begin{center}
\begin{tabular}{|c|l|c|c|} \hline
{\bf Variable Name} & {\bf Description} & {\bf Type} & 
  {\bf Number of categories} \\
& & & {\bf (if applicable)} \\ \hline
Year & The year of flights & categorical & 22 (1987 to 2008) \\ \hline
Month & The month of flights & categorical & 12 \\ \hline
DayOfWeek & The day of week of flights & categorical & 7 \\ \hline
DepTime & The departure time of flights & numeric & NA \\
&  (minutes after midnight) &  & \\ \hline
DepDelay & The departure delay of flights & numeric & NA \\ 
& (minutes) & &  \\ \hline
\end{tabular}
\end{center}
\caption{Variables that will be considered for the covariance matrix stability
benchmark.}
\label{tab:cov_vars}
\end{table}

The $12\,$GB Airline On-time data set will likely not be 
considered ``big'' to many readers. Papers such as 
\cite{Kane2013} have shown how the data set can be explored and
analyzed on relatively modest hardware. However, in designing the
benchmarks two principles were considered before sheer data size.
First, the data set is publicly available. The code included in the
Supplemental Material of this paper is capable of downloading the
data set and running the benchmarks. Users are encouraged to engage
the data themselves and perform their own analyses. Second, the 
data set is large enough to investigate the scatter matrix concordance
properties along with the scaling behavior
of the various regression techniques described in this paper.
Together, the data set and the code available with this paper
provide a set of accessible and reproducible benchmarks that form
a basis for instruction and subsequent research.

The benchmarks presented in the next section were written in the \proglang{R}
programming environment \citep{Rcore} and can be run on a single machine
sequentially, a single machine in parallel, or on a cluster of machines
using the \pkg{foreach} \citep{foreach} package, which provides a concurrent 
interface to a number of different parallel computing technologies for 
embarrassingly parallel challenges. The implementation also makes use of the
\pkg{iterators} and \pkg{itertools} \citep{iterators, itertools} packages
thereby decoupling data access from retrieval and management. 
The most accessible and straightforward
implementation approaches are used to illustrate
the methodological principles
of the models and their implementation.  
The benchmark implementation is flexible enough to be deployed to any 
number of different data management, communication protocol, memory, and 
processing configurations. It is also easily modifiable to accommodate
alternative technologies and methodologies.

\section{Benchmark Results}

This section provides benchmarks exploring the convergence of the
concordance to one on an increasingly larger subset of the Airline On-time
Data set alongside the convergence of the slope coefficients of a Generalized
Linear Model (GLM) of the same subsets. These two sets of benchmarks
establish an empirical connection between regression and concordance. However,
it should also be noted that since the concordance approach does not 
distinguish between independent and dependent variables it provides
a diagnostic not only for the regression but for any regression involving
the same variables. Furthermore, since the concordance calculation for
a subset of the variables can be found directly from the corresponding
scatter matrices a diagnostic easily be calculated for any regression involving
any subset of the variables.


\subsection{Random Sampling Similarity Convergence} \label{sect:rssc}

The first set of benchmarks take random samples of varying 
sizes from 10 to 5000 from the Airline On-time data set and calculates the 
overlapping and non-overlapping concordance between the random subset and 
the entire data set. The overlapping and non-overlapping values
were equal up to seven decimal places and so only the overlapping
concordance is reported.

\begin{figure}
\includegraphics[width=\textwidth]{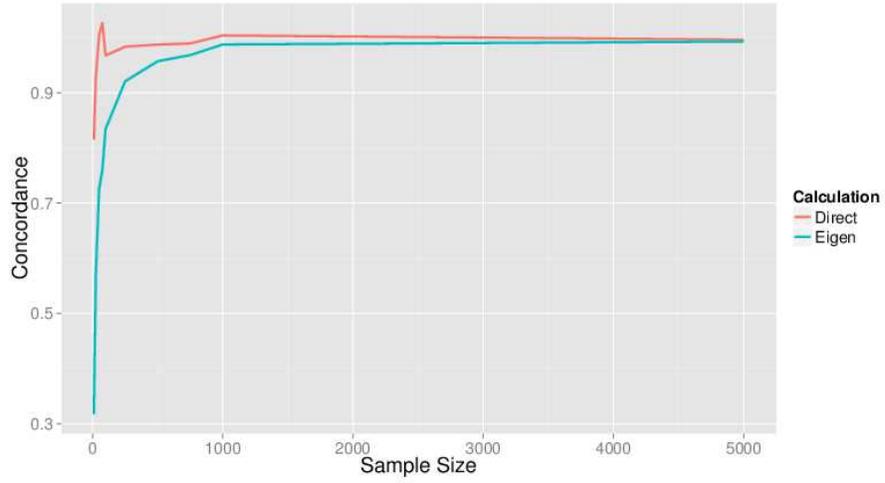}
\caption{Convergence of concordance to one in the 
number of samples using random sampling.}
\label{fig:SimilarityConvergence}
\end{figure}

Figure \ref{fig:SimilarityConvergence} shows the convergence of the
concordance to the value of one when concordance is calculated using 
both Equation \ref{eqn:rel_stability} directly along with the 
trace-equivalent version in \ref{eqn:trace_equivalence}. The plot 
shows that both versions capture the variance-covariance structure 
after only several thousand samples (out of a total of approximately 
120 million). Furthermore, it can be seen that the trace-equivalent version
is slightly smaller than the direct calculation. This is likely because
the direct calculation is more sensitive to noise in the correlation
terms of the scatter matrix and this may account for the overshoot
when the sample size is zero and the small overshoot when the sample size is 
1,000.

\begin{figure}
\includegraphics[width=\textwidth]{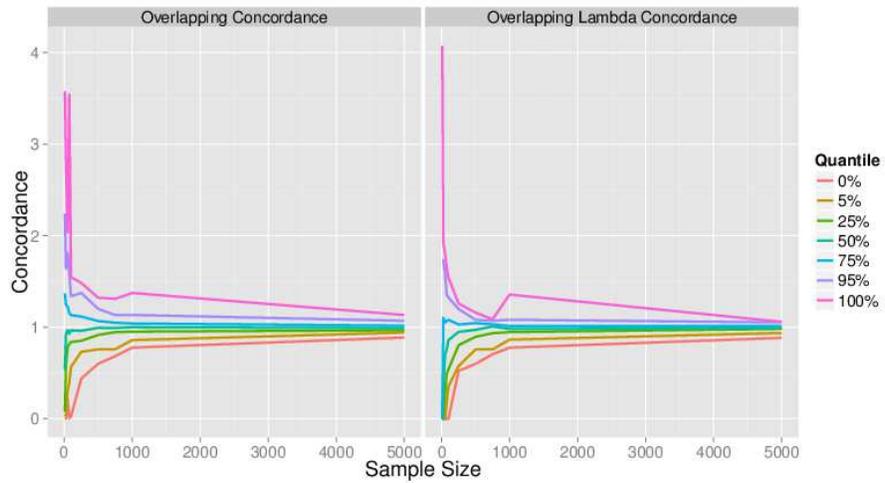} 
\caption{The distribution of overlapping and non-overlapping concordance 
values in the number of samples using random sampling.}
\label{fig:SimilarityQuantileSNonOverlap}
\end{figure}

Each concordance-distribution derivation
reduces to the average of $d$ similarly 
distributed random variables. The terms summed in the
concordance measure can therefore be thought of as samples, and 
their distribution can be examined. This distribution information is shown
in Figure \ref{fig:SimilarityQuantileSNonOverlap} with varying sample size.
The plot again shows little difference between the direct calculation
and its trace-equivalent, especially after a few thousand random samples
are taken and the variance-covariance structure becomes known. The 
trace equivalent does appear to be slightly positively skewed when the number
of samples is smaller and then appears to be essentially symmetric, like
the direct calculation of the concordance. The figure also indicates
that the concordance is not heavy-tailed. The 95\% percentile converges
relatively quickly in the number of samples.

\subsection{Non-Random Sampling Similarity Convergence} \label{sect:nonrandom}

The second set of benchmarks take contiguous samples, starting
at the beginning of the Airline On-time data set, varying the size
from 10 to 1,000,000 to empirically determine the effect of not using
randomized subsets when calculating concordance. Once again, the
overlapping and non-overlapping concordance values were essentially
identical and only the overlapping concordance is reported.

\begin{figure}
\includegraphics[width=\textwidth]{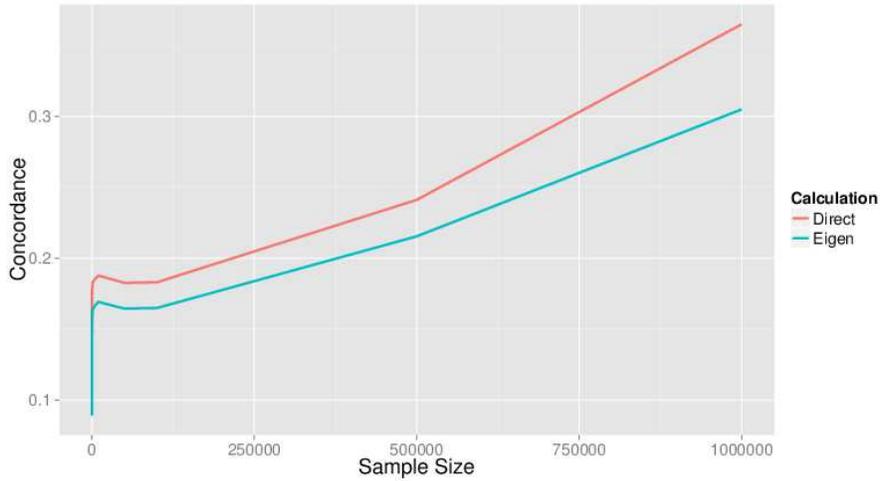}
\caption{Convergence of concordance to one in the 
number of samples using non-random sampling.}
\label{fig:SimilarityConvergenceNonRandom}
\end{figure}

Figure \ref{fig:SimilarityConvergenceNonRandom} shows the convergence of
the concordance values, once again using Equations \ref{eqn:rel_stability}
and \ref{eqn:trace_equivalence}. Where both versions had nearly converged
to one after only a few thousand samples, non-random sampling requires
millions, with the trace-equivalent concordance lagging the direct version.

\begin{figure}
\includegraphics[width=\textwidth]{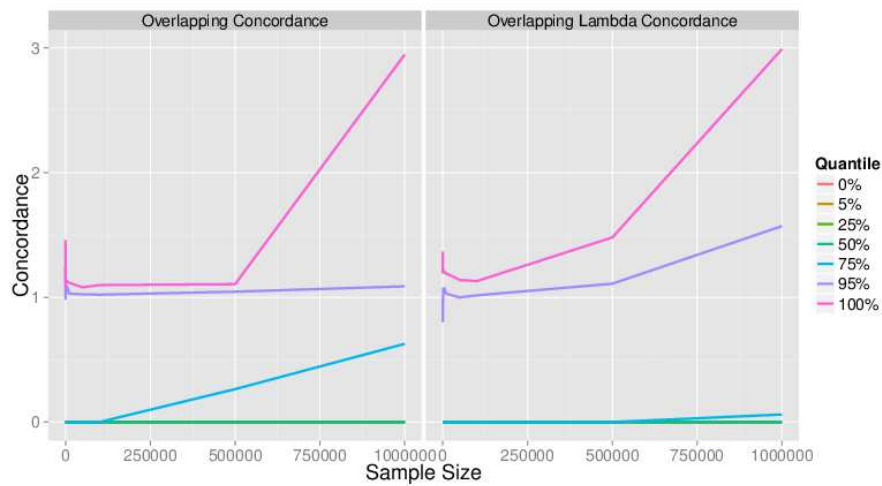} 
\caption{The distribution of overlapping and non-overlapping similarities in 
the number of samples using non-random sampling.}
\label{fig:SimilarityQuantileSNonOverlapNonRandom}
\end{figure}

Convergence in concordance is slow in the non-random case and
this may be due to the fact that 
that the model matrix corresponding to the factor expansion does not 
include a sample of all of the information from the \code{Year}
variable based on Figure \ref{fig:SimilarityConvergenceNonRandom}.
The poor convergence characteristics extend to the distribution of 
the concordance, as shown in Figure \ref{fig:SimilarityConvergenceNonRandom}
thereby underscoring the random sampling in order to achieve representative
subsets.

\subsection{Concordance and GLM Convergence}

The third set of benchmarks compares the concordance with the log mean
square error between the slope coefficient estimates fitted using a 
random subset of the data and estimates fitted using the entire data set.
The variable relationship under investigation is the following 
logistic regression:
\begin{equation*}
Late \sim Year + Month + Day of Week + Departure Time + Departure Delay
\end{equation*}
where a flight is ``$Late$'' if its arrival is at least 30 minutes 
after its scheduled arrival. The other variables are described in 
Table \ref{tab:cov_vars}.
For each subset size, the procedure was repeated 10 times and the
average concordance and log MSE of the coefficients were recorded.

\begin{figure}
\includegraphics[width=\textwidth]{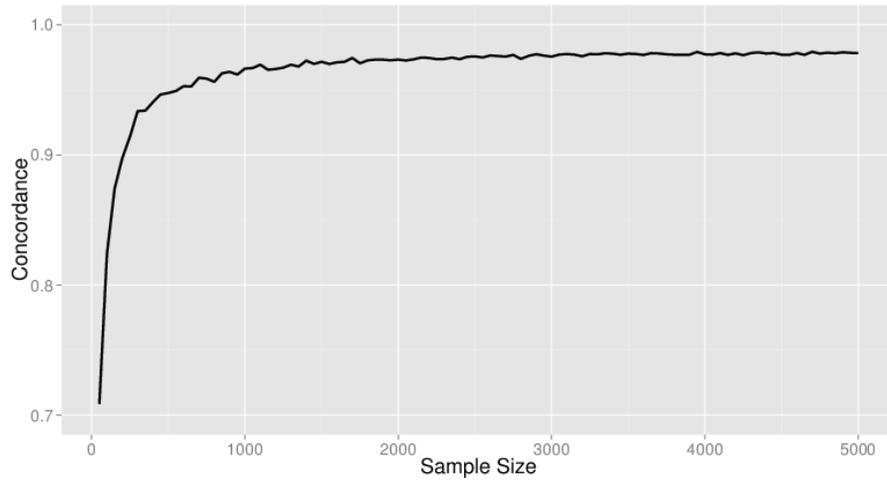}
\caption{Convergence of concordance to one in the 
number of samples using random sampling.}
\label{fig:concordance_glm}
\end{figure}

\begin{figure}
\includegraphics[width=\textwidth]{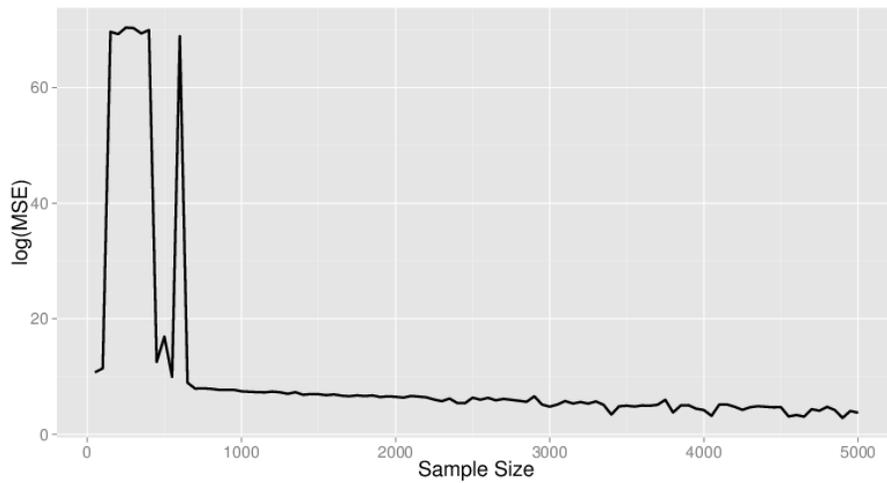}
\caption{Convergence of log MSE between slope coefficient estimates
using random subsets and the entire Airline On-time data set.}
\label{fig:error_glm}
\end{figure}

Figures \ref{fig:concordance_glm} and \ref{fig:error_glm} show the 
convergence behavior of the concordance to one and log MSE of
the slope coefficients to zero respectively. When the number of samples
is small with respect to the number of columns in the design matrix
the concordance is low indicating that the variance-covariance structure
is not well-represented and the log MSE of the slope coefficients 
is high indicating that the estimates of the coefficients is poor.
These values quickly increase and decrease respectively until approximately
500 samples where both the concordance and log MSE slope decreases.

The slow convergence after 500 samples of both measurements implies that
for a precise estimate of the slope coefficients, and corresponding concordance
value close to one a larger subset will need to be fitted. When the subset
is 1,200,000 samples (about 1\% of the total data size) the concordance
values is 0.9820691 and the log MSE is -2.788843. At 12,000,000 samples
(10\% of the data) the concordance value is 0.9822198 and the log MSE is
-4.902919. This ``slow'' convergence in the number of samples may indicative
of complex relationships among the variables that require larger samples
to capture.

\section{Conclusions} \label{sect:conclusion}

Regression models depend directly on the model design matrix and its
properties. This paper attempts to bridge the ``small-data'' and asymptotic
behavior of regression models by proposing a simple measure of concordance 
between a design matrix and a subset
of its rows that estimates how well a subset captures the variance-covariance
structure of increasingly large data sets.
We illustrate the use of this measure in a
heuristic method for selecting row partition sizes that balance statistical and
computational efficiency goals in real-world problems.

Our future work in this area will focus on data fusion. In many cases
it may be desirable to combine two data sources with analogous measurements
to increase the power of statistical experiments. Concordance provides a
distance between the variance-covariance structure with known distributional
characteristics. Tests for equivalence can then be used to rigorously assess the
appropriateness of combining data sources thereby allowing practitioners
to make better use of existing, potentially under-powered data.

\section*{Acknowledgments}

A portion of this research is based on research sponsored by DARPA under award
FA8750-12-2-0324.  The U.S. Government is authorized to reproduce and
distribute reprints for Governmental purposes notwithstanding any
copyright notation thereon.

\section*{Disclaimer}
The views and conclusions contained herein are those of the authors and
should not be interpreted as necessarily representing the official policies
or endorsements, either expressed or implied, of DARPA or the U.S.
Government.

\bibliography{scatter_matrix_sharing}
\bibliographystyle{jss}

\end{document}